\documentclass[letterpaper, 10pt, conference]{ieeeconf}      % Use this line for a4 paper

\IEEEoverridecommandlockouts    
\usepackage[utf8]{inputenc}
\usepackage[english]{babel}
\usepackage[normalem]{ulem}
\usepackage{amsmath}

\DeclareMathOperator*{\argmax}{arg\,max}
\usepackage{amssymb}
\usepackage{amsfonts}
\usepackage{centernot}

\usepackage{fouriernc}
\usepackage{algorithm}
\usepackage{algorithmic}
\usepackage{booktabs}
\usepackage{graphicx}
\usepackage{subfig}
\usepackage{float}
\renewcommand{\algorithmicrequire}{\textbf{Input:}}

\usepackage{hyperref}
\newtheorem{thm}{Theorem}

\newtheorem{df}{Definition}

\newtheorem{rem}{Remark}
\newtheorem{prop}{Proposition}

\usepackage{amssymb}
\newcommand{\norm}[1]{\left\lVert#1\right\rVert}

\newenvironment{varalgorithm}[1]
  {\algorithm\renewcommand{\thealgorithm}{#1}}
  {\endalgorithm}

\usepackage{amssymb}
\title{Potential-Based Advice for Stochastic Policy Learning}%$^*$}
\author{Baicen Xiao$^{1}$, Bhaskar Ramasubramanian$^{1}$, Andrew Clark$^{2}$,\\ Hannaneh Hajishirzi$^{1}$, Linda Bushnell$^{1}$, and Radha Poovendran$^{1}$% <-this % stops a space
\thanks{$^{1}$University of Washington, Seattle, WA 98195, USA. %\newline
       {\tt\small \{bcxiao, bhaskarr, hannaneh, lb2, rp3\}@uw.edu}}%
\thanks{$^{2}$Worcester Polytechnic Institute, Worcester, MA 01609, USA.  %\newline
        {\tt\small aclark@wpi.edu}}%
}
\date{}
\begin{document}
	\maketitle
\begin{abstract}  % put your abstract here!
This paper augments the reward received by a reinforcement learning agent with potential functions in order to help the agent learn (possibly stochastic) optimal policies. 
We show that a potential-based reward shaping scheme is able to preserve optimality of stochastic policies, and demonstrate that the ability of an agent to learn an optimal policy is not affected when this scheme is augmented to soft Q-learning. 
We propose a method to impart potential-based advice schemes to policy gradient algorithms. 
An algorithm that considers an advantage actor-critic architecture augmented with this scheme is proposed, and we give guarantees on its convergence. 
Finally, we evaluate our approach on a puddle-jump grid world with indistinguishable states, and the continuous state and action mountain car environment from classical control.	 
Our results indicate that these schemes allow the agent to learn a stochastic optimal policy faster and obtain a higher average reward.
\end{abstract}
\section{Introduction}
Reinforcement learning (RL) is a framework that allows an agent to complete tasks in an environment, even when a model of the environment is not known. 
The agent `learns' to complete a task by maximizing its expected long-term reward, where the reward signal is supplied by the environment. 
RL algorithms have been successfully implemented in many fields, including robotics \cite{hafner2011reinforcement, lillicrap2016continuous}, and games \cite{mnih2015human, silver2016mastering} . 
However, it remains difficult for an RL agent to master new tasks in unseen environments. 
This is especially true when the reward given by the environment is sparse/ significantly delayed.

It may be possible to guide an RL agent towards more promising solutions faster, if it is equipped with some form of \emph{prior knowledge} about the environment. 
This can be encoded by modifying the reward signal received by the agent during training. 
%Reward modification via human feedback has been used in \cite{thomaz2006reinforcement}-\cite{knox2010combining} to interactively shape an agent's response so that it learns a desired behavior. 
%However, frequent human supervision is usually costly and may not possible in every situation. 
%
However, the modification must be carried out in a principled manner, since providing an additional reward at each step might distract the agent from the true goal \cite{randlov1998learning}. 
%The modification of the reward structure should therefore be carried out in a principled manner. 
Potential-based reward shaping (PBRS) is one such method that augments the reward in an environment specified by a Markov Decision Process (MDP) with a term that is a difference of \emph{potentials} \cite{Ng1999policy}. 
This method is attractive since it easily allows for the recovery of optimal policies, while enabling the agent to learn these policies faster.
%that are functions of the states of the environment. 
%The authors of \cite{Ng1999policy} proved that a static potential-based function preserved the optimality of deterministic policies. 
%This property was extended to dynamic potential-based functions in \cite{Devlin2012dynamic}. 
%The authors of \cite{Wiewiora2003init} showed that when an agent learned a policy using 
%%TD methods, such as 
%Q-learning,  
%%and Sarsa-learning, 
%applying PBRS at each training step was equivalent to initializing the Q-function with the potentials. 
%They studied value-based methods, but restricted their focus to learning deterministic policies. 
%The authors of \cite{Devlin2012express} demonstrated a method to transform a reward function into a potential-based function during training. 

Potential functions are typically functions of states. 
This could be a limitation, since in some cases, such a function may not be able to encode all information available in the environment. %the available information. 
To allow for imparting more information to the agent, a potential-based advice (PBA) scheme was proposed in \cite{Wiewiora2003principled}. 
The potential functions in PBA include both states and actions as their arguments. 

To the best of our knowledge, PBRS and PBA schemes in the literature \cite{Ng1999policy, Wiewiora2003principled, Devlin2012dynamic} assume that an optimal policy is deterministic. 
This will not always be the case, since an optimal policy might be a stochastic policy. 
This is especially true when there are states in the environment that are partially observable or indistinguishable from each other. 
Moreover, the aforementioned papers limit their focus to discrete state and action spaces.
%Moreover, a large part of the existing literature looks at reward-shaping in the context of discrete state and action spaces \cite{Ng1999policy, Wiewiora2003principled}. 

In this paper, we study the addition of PBRS and PBA schemes to the reward, in settings where: \emph{i)} the optimal policy will be stochastic, and \emph{ii)} state and action spaces may be continuous. % the optimal policy will be stochastic. 
We additionally provide guarantees on the convergence of an advantage actor-critic architecture that is augmented with a PBA scheme. 
% \cite{Ng1999policy, Wiewiora2003principled}. 
%In comparison, our approach can be applied to continuous state and action spaces. 
%
We make the following contributions:
\begin{itemize}
%\item We show that potential-based reward shaping (PBRS) preserves optimality of stochastic policies.
%The optimal stochastic policy is still preserved by PBRS, which is a simple extension of existing results.
%
\item We prove that the ability of an agent to learn an optimal stochastic policy remains unaffected when augmenting PBRS to soft Q-learning.% \cite{haarnoja2017reinforcement}.
%The learnability of soft Q-learning \cite{haarnoja2017reinforcement} is maintained when augmented with PBRS.

\item We propose a technique for adapting PBA in policy-based methods, in order to use these schemes in environments with continuous state and action spaces. 

\item We present an Algorithm, \textbf{AC-PBA}, describing an advantage actor-critic architecture augmented with PBA, and provide guarantees on its convergence.

\item We evaluate our approach on two experimental domains: a discrete-state, discrete-action \emph{Puddle-jump Gridworld} that has indistinguishable states, and a continuous-state, continuous-action \emph{Mountain Car}. 
%
%\item When simply adopting PBA during learning, unlike the deterministic cases, the optimal stochastic policy cannot be easily recovered at the end of training.
%
%\item To learn an optimal stochastic policy, a novel method of adapting PBA to policy gradient framework is proposed and some theoretical guarantees are presented.
\end{itemize}
%
%\subsection{Outline}

The remainder of this paper is organized as follows: Section \ref{RelWork} presents related work in reward shaping. Required preliminaries to RL, PBRS and PBA is presented in Section \ref{PrelimSection}. 
Section \ref{PBRSSection} presents our results on using PBRS for stochastic policy learning. 
We present a method to augment PBA to policy gradient frameworks and an algorithm detailing this 
%a PBA-augmented advantage actor-critic architecture 
in Section \ref{PBASection}. 
Experiments validating our approach are reported in Section \ref{ExptSection}, and  we conclude the paper in Section \ref{ConclusionSection}.
\section{Related Work}\label{RelWork}

Shaping or augmenting the reward received by an RL agent in order to enable it to learn optimal policies faster is an active area of research. 
Reward modification via human feedback was used in \cite{thomaz2006reinforcement, knox2010combining} to interactively shape an agent's response so that it learned a desired behavior. 
However, frequent human supervision is usually costly and may not possible in every situation. 
%Furthermore, modifying the reward by providing an additional reward to the agent at each step might distract the agent from the true goal \cite{randlov1998learning}. 
A curiosity-based RL algorithm for sparse reward environments was presented in \cite{pathak2017curiosity}, where an intrinsic reward signal characterized the prediction error of the agent as a curiosity reward. 
The reward received by the agent was augmented with a function that represented the number of times the agent had visited a state in \cite{tang2017exploration}. 

Entropy regularization as a way to encourage exploration of policies during the early stages of learning was studied in \cite{williams1991function} and \cite{mnih2016asynchronous}. 
This was used to lead a policy towards states with a high reward in \cite {levine2013guided} and \cite{levine2016end}. 

%Potential-based methods for reward shaping are attractive since they easily allow for the recovery of optimal policies, while enabling the agent to learn these policies faster. 
%The authors of \cite{Ng1999policy} showed that static potential-based functions preserved the optimality of deterministic policies. 
Static potential-based functions were shown to preserve the optimality of deterministic policies in \cite{Ng1999policy}. 
This property was extended to dynamic potential-based functions in \cite{Devlin2012dynamic}. 
The authors of \cite{Wiewiora2003init} showed that when an agent learned a policy using 
%TD methods, such as 
Q-learning,  
%and Sarsa-learning, 
applying PBRS at each training step was equivalent to initializing the Q-function with the potentials. 
They studied value-based methods, but restricted their focus to learning deterministic policies. 
The authors of \cite{Devlin2012express} demonstrated a method to transform a reward function into a potential-based function during training. 
The potential function in PBA was obtained using an `experience filter' in \cite{li2018introspective}.

The use of PBRS in model-based RL was studied in \cite{asmuth2008potential}, and for episodic RL in \cite{grzes2017reward}. 
PBRS was extended to planning in partially observable domains in \cite{eck2016potential}. 
However, these papers only considered the finite-horizon case. 
In comparison, we consider the infinite horizon, discounted cost setting in this paper. 

In control theoretic settings, RL algorithms have been used to establish guarantees on convergence to an optimal controller for the Linear Quadratic Regulator, when a model of the underlying system was not known in \cite{bradtke1993reinforcement, fazel2018global}. 
A survey of using RL for control is presented in \cite{bucsoniu2018reinforcement}. 
OpenAI Gym \cite{brockman2016openai} enables the solving of several problems in classical control using RL algorithms. 
\section{Preliminaries}\label{PrelimSection}
%
%This section gives a brief introduction to RL, PBRS, and PBA.
\subsection{Reinforcement Learning}

An MDP \cite{puterman2014markov} is a tuple $(S,A,\mathbb{T},\rho_0, R)$. 
$S$ is the set of states, $A$ the set of actions, 
$\mathbb{T}:S \times A \times S \rightarrow [0,1]$ encodes $\mathbb{P}(s_{t+1}|s_t,a_t)$, the probability of transition to $s_{t+1}$, given current state $s_t$ and action $a_t$. 
$\rho_0$ is a probability distribution over the initial states. 
$R : S \times A \rightarrow \mathbb{R}$ denotes the reward that the agent receives when transitioning from $s_t$ while taking action $a_t$. 
In this paper, $R < \infty$. 

The goal for an RL agent \cite{sutton2018reinforcement} is to learn a \emph{policy} $\pi$, 
%which is a map from states to actions 
in order to maximize 
%the expected discounted cumulative reward, i.e. $\max~$ 
$J:=\mathbb{E}_{\tau \sim \pi}[\sum_{t=0}^{\infty}\gamma^t R(s_t,a_t)]$. 
Here, $\gamma$ is a discounting factor, and the expectation is taken over the trajectory $\tau=(s_0,a_0,r_0,s_1,\dots)$ induced by policy $\pi$. %$\pi_{\theta}$. 
If $\pi: S \rightarrow A$, the policy is \emph{deterministic}. 
On the other hand, a randomized policy returns a probability distribution over the set of actions, and is denoted $\pi: S \times A \rightarrow [0,1]$. 

The value of a state-action pair $(s,a)$ following policy $\pi$ is represented by the \emph{Q-function}, written $Q^{\pi}(s,a) = \mathbb{E}_{\tau \sim \pi}[\sum_{t=0}^{\infty}\gamma^t R(s_t,a_t)|s_0=s,a_0=a]$. %, which is the expected discounted cumulative return from state-action pair $(s,a)$.
The Q-function allows us to calculate the state value $V^{\pi}(s) = \mathbb{E}_{a \sim \pi}[Q^{\pi}(s,a)]$. 
The advantage of a particular action $a$, over other actions at a state $s$ is defined by $A^{\pi}(s,a) := Q^{\pi}(s,a)-V^{\pi}(s)$.
\subsection{Value-based and Policy-based Methods}

The RL problem has two general solution techniques. 
\emph{Value-based} methods determine an optimal policy by maintaining a set of reward estimates when following a particular policy. 
At each state, an action that achieves the highest (expected) reward is taken. 
Typical value-based methods to learn greedy (determininistic) policies include Q-learning and Sarsa-learning \cite{sutton2018reinforcement}. 
Recently, the authors of \cite{haarnoja2017reinforcement} proposed \emph{soft Q-learning}, which is a value-based method that is able to learn stochastic policies.

In comparison, \emph{policy-based} methods directly search over the policy space \cite{sutton2018reinforcement}. 
Starting from an initial policy, specified by a set of parameters, these methods compute the expected reward for this policy, and update the parameter set according to certain rules to improve the policy. 
Policy gradient \cite{sutton2000policy} is one way to achieve policy improvement. 
This method repeatedly computes (an estimate of) the gradient of the expected reward with respect to the policy parameters. 
Policy-based approaches usually exhibit better convergence properties, and can be used in continuous action spaces \cite{haarnoja2018soft}. 
They can also be used to learn stochastic policies. 
REINFORCE and actor-critic are examples of policy gradient algorithms \cite{sutton2018reinforcement}.
\subsection{PBRS and PBA}

Reward shaping methods augment the environment reward $R$ with an additional reward $F \in \mathbb{R}$, $F< \infty$. 
This changes 
%has the consequence of changing 
the structure of the original MDP $M(=(S,A,\mathbb{T},\rho_0, R))$ to $M'=(S,A,\mathbb{T},\rho_0, R+F)$. 
The goal is to choose $F$ so that an optimal policy for $M'$, $\pi^{*}_{M'}$, is also optimal for the original MDP $M$. 
%However, if the augmented reward $F$ is with an arbitrary structure, then the optimal policy $\pi^{*}_{M'}$ for $M'$ is not necessarily the same as the optimal policy $\pi^{*}_{M}$ for $M$. 
%The authors of \cite{Ng1999policy} proved that \emph{potential-based reward shaping} (PBRS) schemes were able to preserve the optimality of deterministic  policies. 
\emph{Potential-based reward shaping} (PBRS) schemes were shown to be able to preserve the optimality of deterministic policies in \cite{Ng1999policy}.

In PBRS, the function $F$ is defined as a difference of \emph{potentials}, $\phi(\cdot)$. %, where $\phi$ is a function of states. 
Specifically, $F(s_t,a_t,s_{t+1}) := \gamma \phi(s_{t+1}) - \phi(s_t)$. 
Then, the Q-function, $Q^{*}_{M}(s,a)$, of the optimal greedy policy for $M$ and the optimal Q-function $Q^{*}_{M'}(s,a)$ for $M'$ are related by: 
$Q^{*}_{M'}(s,a)= Q^{*}_{M}(s,a) - \phi(s)$.
Therefore, the optimal greedy policy is not changed \cite{Ng1999policy, Devlin2012dynamic}, since:
\begin{align*}
&\pi^{*}_{M'}(s) \in \argmax_{a\in A}~Q^{*}_{M'}(s,a) \\
&\qquad = \argmax_{a\in A}~\big(Q^{*}_{M}(s,a) - \phi(s)\big) = \argmax_{a\in A}~Q^{*}_{M}(s,a).
\end{align*}

The authors of \cite{Wiewiora2003principled} augmented $\phi(s)$ to include action $a$ as an argument. 
They termed this \emph{potential-based advice} (PBA). 
There are two forms-- \emph{look-ahead PBA} and \emph{look-back PBA}-- respectively defined by:
\begin{align}
F(s_{t},a_{t},s_{t+1},a_{t+1}) &= \gamma \phi(s_{t+1},a_{t+1}) - \phi(s_{t},a_{t})\label{lookaheadPBA}\\
F(s_{t},a_{t},s_{t-1},a_{t-1}) &=  \phi(s_{t},a_{t}) - {\gamma}^{-1}\phi(s_{t-1},a_{t-1}).\label{lookbackPBA}
\end{align}

For the look-ahead PBA scheme, the state-action value function for $M$ following policy $\pi$ is given by: 
\begin{align}\label{PBAq}
Q^{\pi}_{M}(s,a) =  Q^{\pi}_{M'}(s,a)+\phi(s,a).
\end{align} 
The optimal greedy policy for $M$ can be recovered from the optimal state-action value function for $M'$ from:
\begin{align}\label{PBAp}
\pi^*_{M}(s_t) &\in \argmax_{a \in A} \big(Q^{*}_{M'}(s_t,a)+\phi(s_t,a)\big).
\end{align}

The optimal greedy policy for $M$ using look-back PBA can be recovered similarly.
\section{PBRS for Stochastic Policy Learning}\label{PBRSSection}

The existing literature on PBRS has focused on augmenting value-based methods to learn optimal deterministic policies. 
In this section, we first show that PBRS preserves optimality, when the optimal policy is stochastic. 
Then, we show that the \emph{learnability} will not be changed when using PBRS in soft Q-learning.%, which is a value-based method for learning stochastic policies.
%
%\subsection{PBRS and Optimality of Stochastic Policies}
%
\begin{prop}\label{PBRSResult}
Assume that the optimal policy is stochastic. 
Then, with $F:=\gamma\phi(s_{t+1})-\phi(s_t)$, PBRS preserves the optimality of stochastic policies.
\end{prop}
\begin{proof}
The goal in the original MDP $M$ was to find a policy $\pi$ in order to maximize:
\begin{align}\label{PBRS_M}
{\pi}_M^* &= \argmax_{\pi}   \mathbb{E}_{\tau \sim \pi}\left[\sum_{t=0}^{\infty}\gamma^t R(s_t,a_t)\right].
\end{align}

In PBRS, the goal is to determine a policy so that:
 \begin{align}
&{\pi}_{M'}^* = \argmax_{\pi} \mathbb{E}_{\tau \sim \pi}\big[\sum_{t=0}^{\infty}\gamma^t \big(R(s_t,a_t)+ F(s_t,a_t,s_{t+1},a_{t+1})\big)\big] \nonumber\\
&= \argmax_{\pi}  \mathbb{E}_{\tau \sim \pi}\big[\sum_{t=0}^{\infty}\gamma^t \big(R(s_t,a_t)+\gamma\phi(s_{t+1})-\phi(s_t)\big)\big] \nonumber\\
% &=\argmax_{\pi} \bigg[\mathbb{E}_{\tau \sim \pi}\big[\sum_{t=0}^{\infty}\gamma^t R(s_t,a_t)+\sum_{t=1}^{\infty}\gamma^t \phi(s_t)-\sum_{t=0}^{\infty}\gamma^t \phi(s_t)\big]\bigg]\nonumber\\
 &= \argmax_{\pi} \bigg[\mathbb{E}_{\tau \sim \pi}\big[\sum_{t=0}^{\infty}\gamma^t R(s_t,a_t)\big]-\mathbb{E}_{\tau \sim \pi}\big[\phi(s_0)\big]\bigg]\nonumber\\
 &=\argmax_{\pi}~ \mathbb{E}_{\tau \sim \pi}\big[\sum_{t=0}^{\infty}\gamma^t R(s_t,a_t)\big]-\int_{s}\rho_0(s)\phi(s)\text{d}s.\label{PBRS_M'}
 \end{align}
The last term in Equation (\ref{PBRS_M'}) is constant, and doesn't affect the identity of the maximizing policy of (\ref{PBRS_M}).
%Comparing Equations \ref{PBRS_M} and \ref{PBRS_M'} gives that training with PBRS is able to preserve the optimal stochastic policy. 
\end{proof}

%In fact, the above result is true even when the potential function is time-varying, like in \cite{Devlin2012dynamic}. 
%Similar to the proof in \cite{Devlin2012dynamic}, it is easy to show that even when the potential function is dynamic, i.e. $\phi(s)$ is time varying, the optimal stochastic policy can be preserved. 
 %
%\subsection{Learnability of soft Q-learning with PBRS}

%The authors of \cite{Wiewiora2003init} showed that when an agent learned a policy using 
%%TD methods, such as 
%Q-learning,  
%%and Sarsa-learning, 
%applying PBRS at each training step was equivalent to initializing the Q-function with the potentials. 
%They studied value-based methods, but were limited to learning greedy policies. 
Next, we examine the effect on learnability when using PBRS with soft Q-learning. 
Soft Q-learning is a value-based method for stochastic policy learning that was proposed in \cite{haarnoja2017reinforcement}. 
Different from Equation (\ref{PBRS_M}), the goal is to maximize both, the accumulated reward, and the policy entropy at each visited state:
\begin{align}\label{SoftQf}
{\pi}_{\text{soft}}^* &= \argmax_{\pi}   \mathbb{E}_{\tau \sim \pi}\big[\sum_{t=0}^{\infty}\gamma^t \big(R(s_t,a_t)+\alpha\mathcal{H}(\pi(\cdot|s_t))\big)\big].
\end{align}

The entropy term $\mathcal{H}(\pi(\cdot|s_t))$ encourages exploration of the state space, and the parameter $\alpha$ is a trade-off between exploitation and exploration. 

Before stating our result, we summarize the soft Q-learning update procedure. 
%briefly revisit how the updation in soft Q-learning is performed. 
From \cite{haarnoja2017reinforcement}, the optimal value-function, $V^*_{\text{soft}}(s_t) $, is given by:
\begin{align}\label{SoftV}
V^*_{\text{soft}}(s_t) = \alpha \log \int_{A}\exp\big(\frac{1}{\alpha}Q^*_{\text{soft}}(s_t,a)\big)\text{d} a. 
\end{align}
The optimal soft Q-function is determined by solving the soft Bellman equation:
\begin{align}\label{SoftQ}
Q^*_{\text{soft}}\big(s_t,a_t\big) = r_t+\gamma \mathbb{E}_{s_{t+1}}\big[V^*_{\text{soft}}(s_{t+1})\big].
\end{align}
The optimal policy can be obtained from Equation (\ref{SoftQ}) as:
\begin{align}\label{Softp}
{\pi}_{\text{soft}}^*(a_t|s_t) = \exp\big(\frac{1}{\alpha}\big(Q^*_{\text{soft}}(s_t,a_t)-V^*_{\text{soft}}(s_t)\big)\big),
\end{align}

In the rest of this Section, we assume both, states and actions are discrete and no function approximator is used. 
We also omit subscripts for $Q_{\text{soft}}$ and $V_{\text{soft}}$, and set $\alpha=1$ for simplicity. 
From Equation (\ref{SoftQ}), and as in Q-learning, soft Q-learning updates the soft Q-function by minimizing the soft Bellman error:
\begin{align}\label{Bellman_error}
\begin{split}
\delta Q_k(s_k,a_k) = r(s_k,a_k) + \gamma V_k(s_{k+1})-Q_k(s_k,a_k),
\end{split}
\end{align}
where $V_k(s_{t+1}) = \log \sum_{a \in A}\exp\big(Q_k(s_{t+1},a)\big)$. 
During training, $\pi_k(a_t|s_t) = \exp\big(Q_k(s_t,a_t)-V_k(s_t)\big)$.
%where $\mathcal{D}_t$ denotes the sample experience set. % $\mathcal{D} = \{(s^i_t,a^i_t,r^i_t,s^i_{t+1})_{i=1}^N\}$.
With $\lambda$ denoting the learning rate, the Q-function update is given by:
\begin{align}\label{softupdate}
\begin{split}
Q_{k+1}(s_k,a_k) = Q_k(s_k,a_k) +  \lambda\delta Q_k(s_k,a_k).
\end{split}
\end{align}
%where $\lambda$ represents the learning rate.

%We subsequently assume that Q-values in the \emph{soft Q-function}, which are values of state-action pairs in Equation (\ref{SoftQ}), are initialized to potential functions $\phi(s)$. 
The main result of this section shows that the ability of an agent to learn an optimal policy is unaffected when using soft Q-learning augmented with PBRS. 
We define a notion of \emph{learnability}, and use this to establish our claim.

%
%Although the equivalency to Q-function initialization does not hold for soft Q-learning augmented with PBRS, following a similar routine as in \cite{Wiewiora2003principled}, we will show that the learnability of soft Q-learning with PBRS is the same as the learnability of soft Q-learning where only Q-function is initialized by PBRS. 
%Similar to the definition in \cite{Wiewiora2003principled}, the definition of \emph{learnability} is given in Definition \ref{learnability}.
During training, an agent encounters a sequence of states, actions, and rewards that serves as `raw-data' which is fed to the RL algorithm. 
Let $L$ and $L'$ denote two RL agents. 
Let $\mathcal{D}_k = (s_k,a_k,r_k,s_{k+1})$ and $\mathcal{D}'_k = (s'_k,a'_k,r'_k,s'_{k+1})$ denote the experience tuple at learning step $k$ from a trajectory used by $L$ and $L'$, respectively. 
%\textbf{MOTIVATE THIS BETTER}
\begin{df}[Learnability]\label{learnability}
%Let $D$ and $D'$ denote two value-based learning dynamics. 
Denote the accumulated difference in the Q-functions of $L$ and $L'$ after learning for $k$ steps by $\Delta Q_k(s,a)$ and $\Delta Q'_k(s,a)$, respectively. 
%After learning for $t$ steps, the accumulated change in the Q-functions of $D$ and $D'$ are denoted by $\Delta Q_t(s,a)$ and $\Delta Q'_t(s,a)$, respectively. 
Then, given identical sample experiences, (that is, $\mathcal{D}_{k'}=\mathcal{D}'_{k'}$ $\forall k' \leq k$), $L$ and $L'$ are said to have the same learnability if $\Delta Q_{k'}(s,a)=\Delta Q'_{k'}(s,a)$ $\forall k' \leq k  ~ \forall s \forall a$.
\end{df}
\begin{prop}\label{learnability_prop}
Soft Q-learning, with initial soft Q-values $Q(s,a)=Q_0(s,a)$ and augmented with PBRS where state potential is $\phi(s)$, has the same learnability as soft Q-learning without PBRS but with its soft Q-values initialized to $Q(s,a)=Q_0(s,a)+\phi(s)$.
\end{prop}
\begin{proof}
Consider an agent $L$ that uses a PBRS scheme during learning and an agent $L'$ that does not use PBRS, but has its soft Q-values initialized as $Q'_0(s,a):=Q_0(s,a)+\phi(s)$, where $Q_0(s,a)$ is the initial Q-value of $L$.
We further assume that $L$ and $L'$ adopt the same learning rate.
From Definition \ref{learnability}, to show that $L$ and $L'$ have the same learnability, we need to show that the soft Bellman errors $\delta Q_{k}(s_t,a_t)$ and $\delta Q'_{k}(s_k,a_k)$ are equal at each training step $k$, given the same experience sets $\mathcal{D}_k$ and $\mathcal{D}_k'$. 
From Equation (\ref{Bellman_error}), the soft Bellman errors for $L$ and $L'$ can be respectively written as:
	\begin{align*}%\label{beL}
	\delta Q_k(s_k,a_k)& = r(s_k,a_k) + \gamma\phi(s_{k+1})-\phi(s_k)+ \\ 
	&\qquad \gamma V_k(s_{k+1})-Q_k(s_k,a_k)\\
%	\end{align*}
%	\begin{align*}%\label{beL'}
	\delta Q'_k(s'_k,a'_k) &= r(s'_k,a'_k) + \gamma V'_k(s'_{k+1})-Q'_k(s'_k,a'_k).
	\end{align*}
Since $\mathcal{D}_{k'}=\mathcal{D}_{k'}'$ for each $k' \leq k$, comparing $\delta Q'_k(s_k,a_k)$ and $\delta Q_k(s'_k,a'_k)$ is reduced to comparing $\delta Q'_k(s_k,a_k)$ and $\delta Q_k(s_k,a_k)$. 
We show this by induction.
	
At training step $k=0$ there is no update. Thus, $\delta Q_0(s_0,a_0)=\delta Q'_0(s_0,a_0)$. 
Assume that the Bellman errors are identical up to a step $k=K$. 
That is, $\delta Q_k(s_k,a_k)=\delta Q'_k(s_k,a_k)$ $\forall k\leq K$. 
%Assuming up to step $k=K$, the Bellman errors are the same, i.e. $\delta Q_k(s,a)=\delta Q'_k(s,a)$ $\forall k<K$. 
Then, the accumulated errors for the two agents until this step are also identical. 
That is, $\Delta Q_K(s,a)=\Delta Q'_K(s,a) ~ \forall s \forall a$. 
%Therefore, at step $k=K$, accumulated updations for $L$ and $L'$ are also equal, i.e. $\Delta Q_K(s,a)=\Delta Q'_K(s,a)$. 
Consider training step $k=K+1$.
%First we need to calculate $V_K(s_{K+1})$ and $V'_K(s_{K+1})$:
The state values at this step are: 
$V_K(s_{K+1}) = \log \sum_{a \in A}\exp\big[Q_0(s_{K+1},a)+\Delta Q_K(s_{K+1},a)\big]$ and 
$V'_K(s_{K+1})=\log \sum_{a \in A}\exp\big[Q_0(s_{K+1},a)+\phi(s_{K+1})+\Delta Q'_K(s_{K+1},a)\big]$ respectively.
	The Bellman errors at $k=K+1$ are:
	\begin{align*}
	%\begin{split}
	\delta Q_{K+1}(s_K,a_K) = r(s_K,a_K) + &\gamma\phi(s_{K+1})-\phi(s_K)\\
	&+\gamma V_K(s_{K+1})-Q_K(s_K,a_K)\\
	= r(s_K,a_K) + \gamma\phi(s_{K+1})&-\phi(s_K)+\gamma V_K(s_{K+1})\\
	&-Q_0(s_K,a_K)-\Delta Q_K(s_K,a_K)
	%\end{split}
	\end{align*}
	\begin{align*}
	%\begin{split}
	&\delta Q'_{K+1}(s_K,a_K) = r(s_K,a_K) + \gamma V'_K(s_{K+1})-Q'_K(s_K,a_K)\\
	&=r(s_K,a_K) + \gamma V'_K(s_{K+1}) -Q_0(s_K,a_K)-\phi(s_K)-\Delta Q'_K(s_K,a_K)\\
%	\end{align*}
%	\begin{align*}
	&=\delta Q_{K+1}(s_K,a_K)-\gamma\phi(s_{K+1})+\gamma (V'_K(s_{K+1})-V_K(s_{K+1}))\\
	&=\delta Q_{K+1}(s_K,a_K)-\gamma\phi(s_{K+1})+\gamma\phi(s_{K+1})\\
	&=\delta Q_{K+1}(s_K,a_K).
	%\end{split}
	\end{align*}
It follows that $\Delta Q_{K+1}(s,a)=\Delta Q'_{K+1}(s,a) ~ \forall s \forall a$.
\end{proof}
%
%\begin{rem}
%	It should be noted that if $\pi_L$ is equivalent to $\pi_{L'}$, then assuming $\mathcal{D}=\mathcal{D}'$ is reasonable when $N$ is large. 
%	And at the begining of training, it is easy to see that $\pi_L$ is equivalent to $\pi_{L'}$.
%\end{rem}
\begin{rem}
	If the Q-function is represented by a function approximator (as is typical for continuous action spaces), then Proposition \ref{learnability_prop} may not hold. 
	This is because the Q-function in this scenario is updated using gradient descent, instead of Equation (\ref{softupdate}). 
Gradient descent is sensitive to initialization. 
Thus, different initial values will result in different updates of the Q-function.
%And different initialization can result in different gradient with respect to the parameters of Q-function.
\end{rem}
\section{PBA for Stochastic Policy Learning}\label{PBASection}

Although PBRS can preserve the optimality of policies in several settings, it suffers from the drawback of being unable to encode richer information, such as desired relations between states and actions. 
The authors of \cite{Wiewiora2003principled} proposed \emph{potential-based advice} (PBA), a scheme that augments the potential function by including actions as an argument together with states. 
In this section, we show that while using PBA, recovering the optimal policy can be difficult if the optimal policy is stochastic. 
Then, we propose a novel way to impart prior information in order to learn a  stochastic policy with PBA. 
\subsection{Stochastic policy learning with PBA}
%Recall that with PBA, the optimal policy is supposed to optimize the accumulated reward $Q_{M'}(s,a)$ $\forall s\in S$ and $\forall a \in A$:
%\begin{align}
%\begin{split}
%Q^{\pi}_{M'}(s,a) &= \mathbb{E}_{\pi}\big[\sum_{t=0}^{\infty}\gamma^t \big(R(s_t,a_t)+F_t\big)|s_0=s,a_0=a\big]\\ 
%&= Q^{\pi}_{M}(s,a)-\phi(s,a)
%\end{split}
%\end{align}
%It is easy to see that the optimal policy also optimizes $Q_{M}(s,a)$ $\forall s\in S$ and $\forall a \in A$, which is the accumulated reward for the original MDP $M$. And the optimal Q-function can be recovered by adding back the potential term:
%\begin{align}
%\begin{split}
%Q^{*}_{M}(s,a) = Q^{*}_{M'}(s,a)+\phi(s,a)
%\end{split}
%\end{align}

%For example, if some states are partially observable, then we should use belief to denote which state the agent is in and for this case the optimal policy is stochastic.  
%When the optimal policy is stochastic, first we can see that the optimal stochastic policy $\pi^{*}_M$ for $M$ can be different from the optimal stochastic policy $\pi^{*}_{M'}$ for $M'$. 
Assume that we can compute $Q^{*}_{M}(s,a)$, the optimal value for state-action pair $(s,a)$ in MDP $M$. 
The optimal stochastic policy for $M$ is $\pi^*_M = \argmax_{\tau \sim \pi}\mathbb{E}_{\pi}\big[Q^{*}_{M}(s,a)\big]$. 
From Equation (\ref{PBAq}), the optimal stochastic policy for the modified MDP $M'$ that has its reward augmented with PBA is given by $\pi^*_{M'} = \argmax_{\pi}\mathbb{E}_{\tau \sim \pi}\big[Q^{*}_{M}(s,a)-\phi(s,a)\big]$. 
Without loss of generality, $\pi^*_M \neq \pi^*_{M'}$. 
If the optimal policy is deterministic, then the policy for $M$ can be recovered easily from that for $M'$ using Equation (\ref{PBAp}). 
However, when it is stochastic, we need to average over trajectories in the MDP, which makes it difficult to recover the optimal policy for $M$ from that of $M'$. 
%Furthermore, when augmented with PBA, unlike the case where the optimal policy is greedy (deterministic) and we can directly recover the policy from the Q-function as shown in Equation (\ref{PBAp}), it is generally not feasible to recover an optimal stochastic policy from the Q-function.

In the sequel, we will propose a novel way to take advantage of PBA in the policy gradient framework in order to directly learn a stochastic policy.
\subsection{Imparting PBA in policy gradient}

Let $J_M(\theta)$ denote the value of a parameterized policy $\pi_{\theta}$ in MDP $M$. 
That is, $J_M(\theta) = \mathbb{E}_{\tau \sim \pi_{\theta}}\left[\sum_{t=0}^{\infty}\gamma^t R(s_t,a_t)\right]$. 
Following the policy gradient theorem \cite{sutton2018reinforcement}, and defining $G(s_t,a_t):=\sum_{i=t}^{i=\infty}\gamma^{i-t}r_i$, the gradient of $J(\theta)$ with respect to the parameter $\theta$ is given by:
\begin{align}\label{REINFORCE}
\nabla_{\theta}J_M(\theta) = \mathbb{E}_{\tau \sim \pi_{\theta}}\big[G(s_t,a_t)\nabla_{\theta}\log\pi_{\theta}(a_t|s_t)\big].
\end{align}
%where $G(s_t,a_t):=\sum_{i=t}^{i=\infty}\gamma^{i-t}r_i$. 
Then, $\mathbb{E}_{\tau \sim\pi_{\theta}}\big[G(s_t,a_t)\big]=Q^{\pi_{\theta}}(s_t,a_t)$. 

%In the following, we mainly use look-ahead PBA to illustrate the ideas. 
REINFORCE \cite{sutton2018reinforcement} is a policy gradient method that uses Monte Carlo simulation to learn $\theta$, where the parameter update is performed only at the end of an episode (a trajectory of length $T$). 
If we apply a look-ahead PBA scheme as in Equation (\ref{lookaheadPBA}) along with REINFORCE, then the total return from time $t$ is given by:
\begin{align}
\begin{split}
G^{a}(s_t,a_t)&=\sum_{i=t}^{i=T}\gamma^{i-t}r_i+\gamma^{T-t}\phi(s_T,a_T)-\phi(s_t,a_t) \\
&=G(s_t,a_t)+\gamma^{T-t}\phi(s_T,a_T)-\phi(s_t,a_t).
\end{split}
\end{align}

%Similarly, if we apply look-back PBA (\ref{lookbackPBA}) to REINFORCE, then $G^{b}(s_t,a_t)=\sum_{i=t}^{i=T}\gamma^{i-t}r_i-\gamma^{-1}\phi(s_{t-1},a_{t-1})$.
%Now, we can see the issue of using look-ahead PBA with REINFORCE. 
Notice that if $G^{a}(s_t,a_t)$ is used in Equation (\ref{REINFORCE}) instead of $G(s_t,a_t)$, then the policy gradient is biased. 
One way to resolve the problem is to add the difference $-\gamma^{T-t}\phi(s_T,a_T)+\phi(s_t,a_t)$ to $G^{a}(s_t,a_t)$. 
However, this makes the learning process identical to the original REINFORCE and PBA is not used. 
While using PBA in a policy gradient setup, it it important to add the term $\phi(s,a)$ so that the policy gradient is unbiased, and also leverage the advantage that PBA offers during learning.
%
%For each updation the PBA is used only once due to most of potentials are cancelled out by summation. And even worse, we need to add back the potential terms to ensure the gradient estimation is unbiased. But adding back the difference will make PBA disppear. For example, $G'(s_t,a_t)+\gamma^{-1}\phi(s_{t-1},a_{t-1})=\sum_{i=t}^{i=T}\gamma^{i-t}r_i$ which is exactly the same as $G(s_t,a_t)$ used by REINFORCE. 

To apply PBA in policy gradient, we turn to temporal difference (TD) methods. 
TD methods update estimates of the accumulated return based in part on other learned estimates, before the end of an episode.
A popular TD-based policy gradient method is the actor-critic framework \cite{sutton2018reinforcement}. 
In this setup, after performing action $a_t$ at step $t$, the accumulated return $G(s_t,a_t)$ is estimated by $Q_M(s_t,a_t)$ which, in turn, is estimated by $r_t+\gamma V_M(s_{t+1})$. 
It should be noted that the estimates are unbiased.

When the reward is augmented with look-ahead PBA, the accumulated return is changed to $Q_{M'}(s_t,a_t)$, which is estimated by $r_t+\gamma\phi(s_{t+1},a_{t+1})-\phi(s_t,a_t)+\gamma V_{M'}(s_{t+1})$. 
From Equation (\ref{PBAq}), at steady state, $Q_M(s_t,a_t)-Q_{M'}(s_t,a_t)=\phi(s_t,a_t)$. 
Intuitively, to keep policy gradient unbiased when augmented with look-ahead PBA, we can add $\phi(s_t,a_t)$ at each training step. 
In other words, we can use $r_t+\gamma\phi(s_{t+1},a_{t+1})+\gamma V_{M'}(s_{t+1})$ as the estimated return. 
It should be noted that before the policy reaches steady state, adding $\phi(s_t,a_t)$ at each time step will not cancel out the effect of PBA. 
This is unlike in REINFORCE, where the addition of this term negates the effect of using PBA. 
In the advantage actor-critic, an advantage term is used instead of the Q-function in order to reduce the variance of the estimated policy gradient. 
In this case also, the potential term $\phi(s_t,a_t)$ can be added in order to keep the policy gradient unbiased. 
%Similarly, in advantage actor-critic, where advantage is used instead of Q-function to lower the variance of the estimated policy gradient, the potential term $\phi(s_t,a_t)$ can be added back to the advantage $A_{M'}(s_t,a_t)$. 
%
%In the following, we mainly use look-ahead PBA to illustrate how to apply PBA with actor-critic. The idea of applying PBA is intuitive. 
%
%Specifically, once the policy converges, the critic is able to reflect the true value of any state-action pair. If look-ahead PBA is used during learning, then the critic value $r_t++\gamma V(s_{t+1})$ should be equal to the value $Q_{M'}(s_t,a_t)$. We can recover the true value by adding back a potential term, i.e., $Q_{M}(s_t,a_t) = Q_{M'}(s_t,a_t)+\phi(s_{t},a_{t})$. Then $Q_{M}(s_t,a_t)$ can be used to estimate the true policy gradient.
%The critic uses both true rewards and estimated rewards. In such a way, not only the environment information but also PBA is bootstrapped. The PBA is able to broadcast during learning and will not be cancelled out like in REINFORCE. Therefore, even we will add back the potential term when learning policy, the advantage of PBA is kept.
\begin{varalgorithm}{AC-PBA}
%\floatname{algorithm}{Test 123}
	\caption{: Actor-critic augmented with PBA}
\begin{algorithmic} \label{Algo1}
	\renewcommand{\algorithmicrequire}{\textbf{Input:}}
	\REQUIRE Differentiable policy function $\pi_{\theta}(a|s)$% parametrized by $\theta$
	
	\hspace{8.8mm}Differentiable value function $V^{\omega}(s)$% parametrized by $\omega$
	
	\hspace{8.8mm}Potential-based advice $\phi(s,a)$
	
	\hspace{8.8mm}Maximum episode $T_{max}$
	
	\textbf{Initialization}: \\
	policy parameter $\theta$, value parameter $\omega$, learning rate $\alpha^{\theta}$ and $\alpha^{\omega}$, discount factor $\gamma$, episode counter $T \leftarrow 0$
	\REPEAT
	\STATE initialize state $s_0$, $t \leftarrow 0$
	\REPEAT
	\STATE Sample action $a_t \sim \pi_{\theta}(\cdot|s_t)$
	\STATE Take action $a_t$, observe reward $r_t$, next state $s_{t+1}$
	\STATE $R= \begin{cases}
	0, & \text{if }
	\begin{aligned}[t]
	s_{t+1} \text{ is a terminal state },
	\end{aligned}
	\\
	V^{\omega}(s_{t+1}), & \text{otherwise.}
	\end{cases}$
	
	\IF{use look-ahead advice} 
	\STATE $\delta_t=r_t + \gamma\phi(s_{t+1},a_{t+1})-\phi(s_t,a_t)+\gamma R - V^{\omega}(s_t)$ 
	\STATE Update $\theta \leftarrow \theta + \alpha^{\theta} \big(\delta_t+\phi(s_t,a_t)\big)\nabla_{\theta}\log\pi_{\theta}(a_t|s_t)$
	
	\ELSE %\Comment{(look-back advice)}
	\STATE $\delta_t = r_t + \phi(s_{t},a_{t})-\gamma^{-1}\phi(s_{t-1},a_{t-1})+\gamma R - V^{\omega}(s_t)$
	\STATE Update $\theta \leftarrow \theta + \alpha^{\theta} \delta_t\nabla_{\theta}\log\pi_{\theta}(a_t|s_t)$
	\ENDIF
	
	\STATE Update $\omega \leftarrow \omega - \alpha^{\omega} \delta_t\nabla_{\omega}V^{\omega}(s_t)$
	
	\UNTIL{$s_{t+1}$ is a terminal state}
	\STATE $T \leftarrow T+1$
	\UNTIL{$T>T_{max}$}	
\end{algorithmic}
\end{varalgorithm}

A procedure for augmenting the advantage actor-critic with PBA is presented in Algorithm \ref{Algo1}. 
$\alpha^{\theta}$ and $\alpha^{\omega}$ denote learning rates for the actor and critic respectively. 
When applying look-ahead PBA, at training step $t$, parameter $\omega$ of the critic $V^{\omega}(s)$ is updated as follows:
\begin{align}
\delta^a_t &= r_t + \gamma\phi(s_{t+1},a_{t+1})-\phi(s_t,a_t)+\gamma V^{\omega}(s_{t+1}) - V^{\omega}(s_t)\nonumber \\
\omega &= \omega - \alpha^{\omega} \delta^a_t\nabla_{\omega}V^{\omega}(s_t),\nonumber
\end{align}
where $\delta^a_t$ is the estimation error of the state value after receiving new reward $[r_t + \gamma\phi(s_{t+1},a_{t+1})-\phi(s_t,a_t)]$ at step $t$.
%$\delta^a_t$ can also be viewed as the estimated advantage value to calculate policy gradient for updating the actor.
To ensure an unbiased estimate of the policy gradient, the potential term $\phi(s_t,a_t)$ is added while updating $\theta$ as:
\begin{align}
\theta = \theta + \alpha^{\theta} \big(\delta^a_t+\phi(s_t,a_t)\big)\nabla_{\theta}\log\pi_{\theta}(a_t|s_t).\nonumber
\end{align}
%where $\alpha^{\theta}$ denotes learning rate.
%Since the potential term $\phi(s_t,a_t)$ is added back, the variance of the update may increase if $\phi(s_t,a_t)$ is large. 
%One way to mitigate this is by adding a baseline term $b$ (for eg., $b$ could be the average potential value over all feasible actions at state $s_t$). 
%After adding the baseline, the estimate of the policy gradient is still unbiased \cite{sutton2018reinforcement}. 
%The update of $\theta$ then becomes:
%\begin{align}
%\begin{split}
%\theta = \theta + \alpha \big(\delta^a_t+\phi(s_t,a_t)-b\big)\nabla_{\theta}\log\pi_{\theta}(a_t|s_t)
%\end{split}
%\end{align}

A similar method can be used when learning with look-back PBA. 
In this case, the critic and the policy parameter are updated as follows:
\begin{align}
\delta^b_t &= r_t + \phi(s_{t},a_{t})-\gamma^{-1}\phi(s_{t-1},a_{t-1})+\gamma V^{\omega}(s_{t+1}) - V^{\omega}(s_t)\nonumber\\
\omega &= \omega - \alpha^{\omega} \delta^b_t\nabla_{\omega}V^{\omega}(s_t),\nonumber\\
%\end{align}
%and the policy parameter is updated as:
%\begin{align}
\theta &= \theta + \alpha \big(\delta^b_t+\gamma^{-1}\mathbb{E}\big[\phi(s_{t-1},a_{t-1})|s_t\big]\big)\nabla_{\theta}\log\pi_{\theta}(a_t|s_t)\label{addback}
\end{align}

In fact, the potential term need not be added to ensure an unbiased estimate in this case. 
Then, the policy parameter update becomes:
\begin{align}\label{noaddback}
\theta = \theta + \alpha \delta^b_t\nabla_{\theta}\log\pi_{\theta}(a_t|s_t),
\end{align}
which is exactly the policy update of the advantage actor-critic. 
This is formally stated in Proposition \ref{PBAprop} 
%Intuitively the variance will be lower when learning with look-back PBA. 
\begin{prop} \label{PBAprop}
	When the actor-critic is augmented with look-back PBA, Equations (\ref{addback}) and (\ref{noaddback}) are equal in the sense of expectation. That is
	\begin{align}   
	\mathbb{E}_{(s_t,a_t) \sim \rho^{\pi_\theta}}\big[\big(\delta^b_t+&\gamma^{-1}\mathbb{E}\big[\phi(s_{t-1},a_{t-1})|s_t\big]\big)\nabla_{\theta}\log\pi_{\theta}(a_t|s_t)\big] \nonumber\\
	=	\quad&\mathbb{E}_{(s_t,a_t) \sim \rho^{\pi_\theta}}\big[\delta^b_t\nabla_{\theta}\log\pi_{\theta}(a_t|s_t)\big],
	\end{align} 
	where $\rho^{\pi_\theta}$ is the distribution induced by the policy $\pi_\theta$.
\end{prop}
\begin{proof}
It is equivalent to show that:
\begin{align}   
\mathbb{E}_{(s_t,a_t) \sim \rho^{\pi_\theta}}\big[\mathbb{E}\big[\phi(s_{t-1},a_{t-1})|s_t\big]\nabla_{\theta}\log\pi_{\theta}(a_t|s_t)\big] = 0.
\end{align} 
The inner expectation $\mathbb{E}\big[\phi(s_{t-1},a_{t-1})|s_t\big]$ is a function of $s_t$, policy $\pi_{\theta}$, and transition probability $\mathbb{T}$. 
Denoting this expectation by $f(s_t,\pi_{\theta},\mathbb{T})$, we obtain:
\begin{align}   
&\mathbb{E}_{(s_t,a_t) \sim \rho^{\pi_\theta}}\big[f(s_t,\pi_{\theta},\mathbb{T})\nabla_{\theta}\log\pi_{\theta}(a_t|s_t)\big]\nonumber \\= 
&\mathbb{E}_{s_t \sim
	\rho^{\pi_\theta}}\bigg[\mathbb{E}_{a_t\sim \pi_\theta}\big[f(s_t,\pi_{\theta},\mathbb{T})\nabla_{\theta}\log\pi_{\theta}(a_t|s_t)\big]\bigg]\nonumber \\= 
&\mathbb{E}_{s_t \sim
	\rho^{\pi_\theta}}\bigg[\int_{A}\pi_{\theta}(a_t|s_t)f(s_t,\pi_{\theta},\mathbb{T})\frac{\nabla_{\theta}\pi_{\theta}(a_t|s_t)}{\pi_{\theta}(a_t|s_t)}\text{d}a\bigg] \nonumber\\= 
&\mathbb{E}_{s_t \sim
	\rho^{\pi_\theta}}\bigg[f(s_t,\pi_{\theta},\mathbb{T})\nabla_{\theta}\int_{A}\pi_{\theta}(a_t|s_t)\text{d}a\bigg] = 0.
%&\mathbb{E}_{s_t \sim
%	\rho^{\pi_\theta}}\bigg[f(s_t,\pi_{\theta},\mathbb{T})\nabla_{\theta}1\bigg] = 0.
\end{align}
The last equality follows from the fact that the integral evaluates to $1$, and its gradient is $0$.
\end{proof}

%\textbf{Theoretical convergence guanrantee (to be finished).}
The main result of this paper presents guarantees on the convergence of Algorithm \ref{Algo1} using the theory of `two time-scale stochastic analysis'  \cite{Borkar2000ODE}. Assume that:
%Based on two-time-scale stochastic analysis \cite{Borkar2000ODE}, we can give a theoretical convergence guarantee for actor-critic augmented with PBA, which is presented in Theorem \ref{convergence}. 
%We make the following assumptions:
\begin{itemize}
	\item \textbf{A1}: The value function $V^{\omega}(s)$ belongs to a linear family. That is, $V^{\omega}=\Phi \omega$, where $\Phi \in \mathbb{R}^{|S|\times k}, k <S$ is a known full-rank feature matrix, and $\omega \in \Omega  \subseteq \mathbb{R}^{k}$.% Specifically, $\Phi$ has full rank and $k<|S|$. 
	\item \textbf{A2}: For the set of policies $\{\pi_{\theta}, \theta \in \Theta\subseteq \mathbb{R}^{d}\}$, there exists a constant $C_{\Theta}$ such that $\norm{\nabla_{\theta}\log\pi_{\theta}}_2\leq C_{\Theta}$.
%	\item \textbf{A3}: The rewards $R$ of MDP $M$ and $R+F$ of MDP $M'$ are upper-bounded by a constant $R_{max}$.
	\item \textbf{A3}: Learning rates of the actor and critic satisfy: $\sum_t \alpha_t^{\theta}=\sum_t \alpha_t^{\omega}=\infty$, $\sum_t [(\alpha_t^{\theta})^2+(\alpha_t^{\omega})^2]<\infty$, 
	%$\sum_t (\alpha_t^{\theta})^2<\infty$, $\sum_t (\alpha_t^{\omega})^2<\infty$ 
	$\lim\limits_{t\rightarrow \infty}\frac{\alpha_t^{\theta}}{\alpha_t^{\omega}}=0$.
\end{itemize}
For any probability measure $\mu$ on a finite set $\mathcal{M}$, the $\ell_2$-norm of $f$ with respect to $\mu$ is given by $\norm{f}_{\mu}:=\big[\int_{\mathcal{M}}|f(x)|^2\text{d}\mu(x)\big]^{\frac{1}{2}}$. 
%Let $T^{\pi}_{M'}$ denote an operator respectively defined for look-ahead and look-back PBA as:
%\begin{align*}
%(T^{\pi}_{M'}V)(s) := \mathbb{E}_{\pi}\big[r(s_t,a_t)+F(s_t,a_t,s_{t+1},a_{t+1})+\gamma V(s_{t+1})|s_t=s\big]\\
%(T^{\pi}_{M'}V)(s) := \mathbb{E}_{\pi}\big[r(s_t,a_t)+F(s_t,a_t,s_{t-1},a_{t-1})+\gamma V(s_{t+1})|s_t=s\big]
%\end{align*}
%where $F(s_0,a_0,s_{-1},a_{-1}):=\phi(s_0,a_0)$ for look-back PBA. 
%It can be shown that the operator $T^{\pi}_{M'}$ is a contraction. 
Theorem \ref{convergence} gives a bound on the error introduced as a result of approximating the value function $V_{M'}$ with $V_{M'}^{\omega}$ as in assumption \textbf{A1}. 
This error term is small if the family $\Omega$ is rich. 
In fact, if the critic is updated in batches, a tighter bound can be achieved, as shown in Proposition 1 of \cite{Yang2018finite}. 
Extending the result to the case of online updates is a subject of future work.
\begin{thm}\label{convergence} 
%Let $\mathcal{E}(\Omega,\theta)=\sup\limits_{\omega'} \inf\limits_{\omega}\norm{V_{\omega}-(T^{\pi}_{M'}V_{\omega'})}_{\rho^{\pi_{\theta}}}$.
Let $\mathcal{E}(\theta):=\norm{V^{\omega(\theta)}_{M'}(s)-V^{\pi_{\theta}}_{M'}(s)}_{\rho^{\pi_{\theta}}}$.
%which can be used to bound policy evaluation error \cite{Yang2018finite}.
Then, for any limit point $(\theta^*,\omega^*) :=\lim\limits_{T_{max} \to \infty}(\theta_{T_{max}},\omega_{T_{max}})\}$ of Algorithm \ref{Algo1}, 
%\begin{align}
	$\norm{\nabla_{\theta}J_M(\theta^*)}_2\leq C\mathcal{E}(\theta^*)$.
%\end{align}
%where C is a constant.
\end{thm}
%
%The error $\mathcal{E}(\theta)$ is bounded by the maximum bias introduced by approximating the value function $V^{\pi_\theta^*}_{M'}$. 
%The error term is small if the family $\Omega$ is rich. 
%In fact, based on Proposition 1 in \cite{Yang2018finite}, if the critic is updated in batches and the batch size is large enough, then we can bound the error as follows,
%\begin{align*}
%	\mathcal{E}(\theta)\leq C_1\sup\limits_{\omega'} \inf\limits_{\omega}\norm{V^{\omega}-(T^{\pi}_{M'}V^{\omega'})}_{\rho^{\pi_{\theta}}},
%\end{align*}
%where $C_1$ is some constant.
%For the online updating case, a further bound on the error $\mathcal{E}(\theta)$ is remained as future work.
%Theorem \ref{convergence} means that the limit $\pi(\theta^*)$ obtained in MDP $M'$ is also close to the local optimal policy in the original MDP $M$. 
%This Theorem is adapted from \cite{Yang2018finite} and \cite{Bhatnagar2009} and we give a sketch for the proof in the following.
%
\begin{proof}
We consider only look-ahead PBA. 
The proof for look-back PBA follows similarly. 
Define $F:=F(s,a,s',a')$. 
From assumption \textbf{A3}, the actor is updated at a slower rate than the critic. 
This allows us to fix the actor to study the asymptotic behavior of the critic \cite{Bhatnagar2009}.  
The update dynamics of the critic can be represented by:% the ODE:
\begin{align}\label{ODE_critic}
\dot{\omega}=\mathbb{E}_{ \rho^{\pi_{\theta}}}\big[\delta_{\omega}\nabla_{\omega}V^{\omega}_{M'}(s)\big],
\end{align}
where $\delta_{\omega}=r(s,a) + \gamma\phi(s',a')-\phi(s,a)+\gamma V^{{\omega}}(s') - V^{\omega}(s)$ if look-ahead PBA is applied. 
%The critic is on the faster time-scale, and when it is 
When the critic is 
approximated by a linear function (assumption \textbf{A1}), $\omega$ will converge to $\omega(\theta)$, an asymptotically stable equilibrium of Equation (\ref{ODE_critic}). % when the critic is approximated by a linear function (assumption \textbf{A1}) \cite{Yang2018finite}.
The update of the actor is then:% by the following ODE:
\begin{align}\label{ODE_actor}
\dot{\theta}=\mathbb{E}_{ \rho^{\pi_{\theta}}}\big[\nabla_{\theta}\log \pi_{\theta}(a|s)\big(r(s,a)+F+\gamma V_{M'}^{\omega(\theta)}(s')+\phi(s,a)\big)\big].
\end{align}
Let $\Theta_s$ denote the set of asymptotic stable equilibria in Equation (\ref{ODE_actor}). 
Any $\theta \in \Theta_s$ will satisfy $\dot{\theta}=0$ in Equation (\ref{ODE_actor}). %which satisfies:
%\begin{align}\label{equilibrium}
%\mathbb{E}_{ \rho^{\pi_{\theta}}}\big[\big(r(s,a)+F(s,a,s'&,a')+
%\gamma V_{M'}^{\omega(\theta)}(s')+\nonumber\\
%&\phi(s,a)\big)\nabla_{\theta}\log \pi_{\theta}(a|s)\big] = 0.
%\end{align}
Then, $\{(\theta_t,\omega_t)\}_{t>0}$ will converge to $\{(\theta,\omega(\theta)):\theta \in \Theta_s\}$.

Now, consider the evaluation of $\pi_{\theta}$, $\theta \in \Theta_s$, in the original MDP $M$.
We obtain the following equations:
\begin{align}\label{evaluation}
&\nabla_{\theta}J_M(\theta) =
\mathbb{E}_{ \rho^{\pi_{\theta}}}\big[\nabla_{\theta}\log \pi_{\theta}(a|s)Q^{\pi_{\theta}}_{M}(s,a)\big] \nonumber\\
&= \mathbb{E}_{ \rho^{\pi_{\theta}}}\big[\nabla_{\theta}\log \pi_{\theta}(a|s)\big(Q^{\pi_{\theta}}_{M'}(s,a)+\phi(s,a)\big)\big] \nonumber\\
%& = \mathbb{E}_{ \rho^{\pi_{\theta}}}\big[\nabla_{\theta}\log \pi_{\theta}(a|s)\big(r(s,a)+\mathbb{E}_{ \rho^{\pi_{\theta}}}[F+\gamma V^{\pi_{\theta}}_{M'}(s')|s,a]+\phi(s,a)\big)\big]\nonumber\\
& = \mathbb{E}_{ \rho^{\pi_{\theta}}}\big[\nabla_{\theta}\log \pi_{\theta}(a|s)\big(r(s,a)+F+\gamma V^{\pi_{\theta}}_{M'}(s')+\phi(s,a)\big)\big]. 
\end{align}
Subtracting Equation (\ref{ODE_actor}) from Equation (\ref{evaluation}), and applying the Cauchy-Schwarz inequality to the result yields:
\begin{align*}%\label{gradient_error}
\nabla_{\theta}J_M(\theta)&=\gamma \mathbb{E}_{ \rho^{\pi_{\theta}}}\big[\nabla_{\theta}\log \pi_{\theta}(a|s)\big( V_{M'}^{\omega(\theta)}(s')-V^{\pi_{\theta}}_{M'}(s')\big)\big]\\
\therefore \norm{\nabla_{\theta}J_M(\theta)}_2&\leq \gamma\norm{\nabla_{\theta}\log \pi_{\theta}(a|s)}_{\rho^{\pi_{\theta}}}\norm{V_{M'}^{\omega(\theta)}(s)-V^{\pi_{\theta}}_{M'}(s)}_{\rho^{\pi_{\theta}}}.
\end{align*}
The result follows by applying assumption \textbf{A2}.
%Consider assumption \textbf{A2}, we can obtain:
%\begin{align*}
%\norm{\nabla_{\theta}J_M(\theta)}_2\leq C\mathcal{E}(\theta),
%\end{align*}
%where $C$ is some constant.
\end{proof}
\begin{rem}
%Proposition \ref{PBRSResult} indicates that PBRS preserves optimality of stochastic policies. 
%Therefore, it can be used simultaneously with PBA, as we will report in our experiments.
Look-back PBA could result in better performance compared to look-ahead PBA since look-back PBA does not involve estimating a future action. %estimating a future action is not required in look-back PBA. Secondly, the potential term is not necessarily added back and the policy gradient estimation can have lower variance.  
\end{rem}
\section{Experiments}\label{ExptSection}

Our experiments seek to compare the performance of an actor-critic architecture augmented with PBA and with PBRS with the `vanilla' advantage actor-critic (A2C). 
We consider two setups. 
The first is a \emph{Puddle-Jump Gridworld} \cite{marom2018belief}, where the state and action spaces are discrete. 
%To demonstrate the generalizability of our method to continuous state and action spaces, 
The second environment we study is a continuous state and action space \emph{mountain car} \cite{brockman2016openai}. 

In each experiment, we compare the rewards received by the agent when it uses the following schemes: \emph{i)}: `vanilla' (A2C); \emph{ii)}: A2C augmented with PBRS; \emph{iii)}: A2C with look-ahead PBA; \emph{iv)}: A2C with look-back PBA.
\subsection{Puddle-Jump Gridworld}
\begin{figure}
	\centering
	\includegraphics[width=1.5 in]{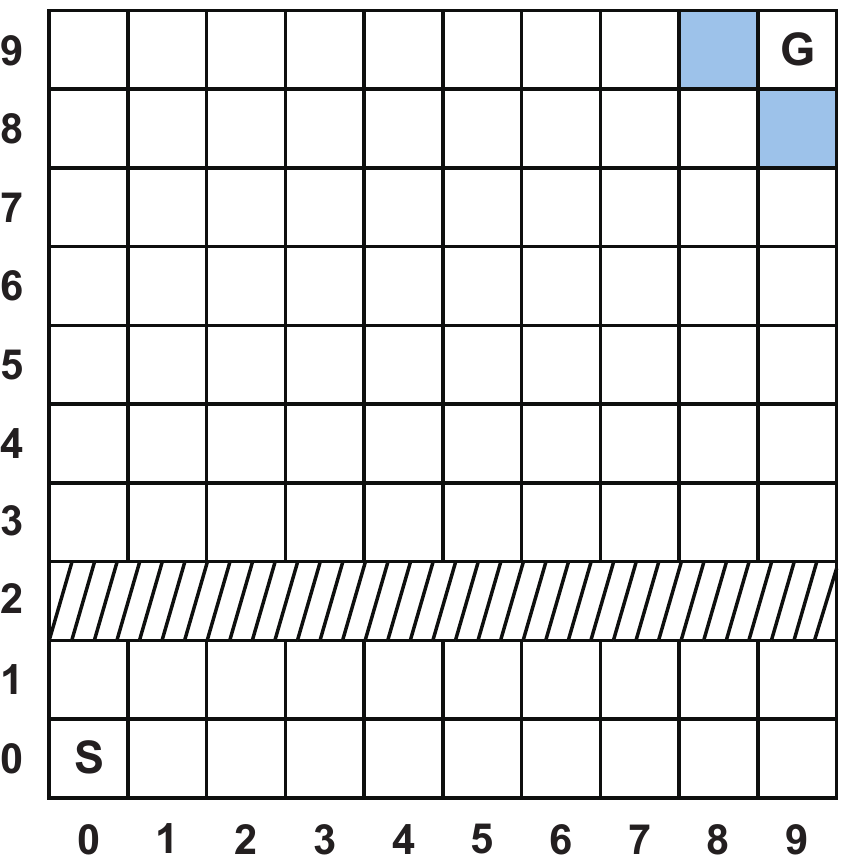}
	\caption{Schematic of the puddle-jump gridworld. The state of the agent is its position $(x,y)$. The shaded row (row $2$) represents the puddle the agent should jump over. The two blue grids denote states that are indistinguishable to the agent. The agent can choose an action from the set $\{up, down, left, right, jump\}$ at each step.}\label{GridWorld}
\end{figure}
%\begin{figure}
%	\centering
%	\includegraphics[width=3.3 in]{cliff_jump.pdf}
%	\caption{Average rewards in cliff-jump gridworld when jump success probability $p_j=0.2$. The baseline is the advantage actor-critic without advice.}\label{cliff_results}
%\end{figure}
%%
%\begin{figure}
%	\centering
%	\includegraphics[width=3.3 in]{cliff_jump2.pdf}
%	\caption{Average reward for the first 100 episodes with respect to the jump success probability $p_j$.}\label{cliff_results2}
%\end{figure}

Figure \ref{GridWorld} depicts the \emph{Puddle-jump gridworld} environment as a 10x10 grid. 
The state space is $s=(x,y)$ denoting the position of the agent in the grid, where $x,y \in \{0,1,\dots,9\}$. 
The goal of the agent is to navigate from the start state $S= (0,0)$ to the goal $G=(9,9)$. 
At each step, the agent can choose from actions in the set $A = \{up, down, left, right, jump\}$. 
%In this environment, there are five actions the agent can choose from at each step, i.e., four cardinal directions and jump action. The agent's state is denoted by its position in the gridworld. 
There is a \emph{puddle} along row $2$ which the agent should jump over. Further, the states $(9,8)$ and $(8,9)$ (blue squares in Figure \ref{GridWorld}) are indistinguishable to the agent. 
As a result, any optimal policy for the agent is a stochastic policy.
%And value-based RL methods such as Q-learning and Sarsa learning cannot be used in this environment. 
 
If the $jump$ action is chosen in rows $3$ or $1$, the agent will land on the other side of the puddle with probability $p_j$, and remain in the same state otherwise. 
This action chosen in other rows will keep the agent in its current state. 
Any action that will move the agent off the grid will keep its state unchanged. 
%If an action is chosen such that the agent will move off the grid, the agent will remain at its current position. 
The agent receives a reward of $-0.05$ for each action, and $+1000$ for reaching $G$.
%A reward of -0.05 is given for each action the agent takes, except that the agent receives a reward of 1000 for reaching the goal.

%In this environment, we also test actor-critic augmented with PBRS, as shown in previous section it can preserve the optimality and it is intuitive to give a potential-based reward for this gridworld. 
When using PBRS, we set $\phi^{PBRS}(s):=u_0$ for states in rows $0$ and $1$,  and $\phi^{PBRS}(s):=u_1$ for all other states. 
We need $u_1>u_0$ to encourage the agent to jump over the puddle. 
Unlike in PBRS, PBA can provide the agent with more information about the  actions it can take. 
We set $\phi^{PBA}(s,a)$ to a `large' value if action $a$ at state $s$ results in the agent moving closer to the goal according to the $\ell_1$ norm, $\big(|G-x|+|G-y|
\big)$. 
We additionally stipulate that $\frac{1}{|A|}\sum_{a\in A}\phi^{PBA}(s,a) = \phi^{PBRS}(s)$. That is, the state potential of PBA is the same as the state potential of PBRS under a uniform distribution over the actions. 
This is to ensure a fair comparison between PBRS and PBA.

In our experiment, we set the discount factor $\gamma=1$. 
Since the dimensions of the state and action spaces is not large, we do not use a function approximator for the policy $\pi$. 
A parameter $\theta_{s,a}$ is associated to each state-action pair, and the policy is computed as:
%\begin{align}
$\pi_{\theta}(a|s)=\frac{\exp(\theta_{s,a})}{\sum_{a\in A}\exp(\theta_{s,a})}$.
%\end{align} 
%The function for state value is table-based, i.e., $w_s=V(s)$. 
We fix $\alpha^{\omega}= 0.001$, and $\alpha^{\theta}= 0.2$ for all cases. 
From Figure \ref{cliff_results}, we observe that the look-back PBA scheme performs the best, in that the agent converges to the goal in \textbf{five times} fewer episodes ($25$ vs. $125$ episodes) than A2C without advice. % results in the fastest convergence to the goal. 
%the performance of look-back PBA is the best in terms of both convergence and variance. 
When A2C is augmented with PBRS, convergence to the goal is slightly faster than without any reward shaping. 
%However, when augmented with look-ahead PBA, the convergence is slower than A2C without advice.
When augmented with look-ahead PBA, in the first few episodes, the reward increases faster than in the case of A2C augmented with PBRS. 
However, this slows down after the early training stages and the policy converges to the goal in about the same number of episodes as a policy trained without advice.
A reason for this could be that during later stages of training, a look-ahead PBA scheme might advise an agent with `bad' actions, leading to bad policies, thereby impeding the progress of learning. 
%One interpretation is that it is not good to believe the value of look-ahead PBA in the early steps of training. 
For example, an action $a_t$ might be a good choice at state $s_t$, but the look-ahead PBA scheme might indicate that $a_t$ is bad, due to a poor estimate of the future action $a_{t+1}$.
\begin{figure}
	\centering
	\includegraphics[width=2.8 in]{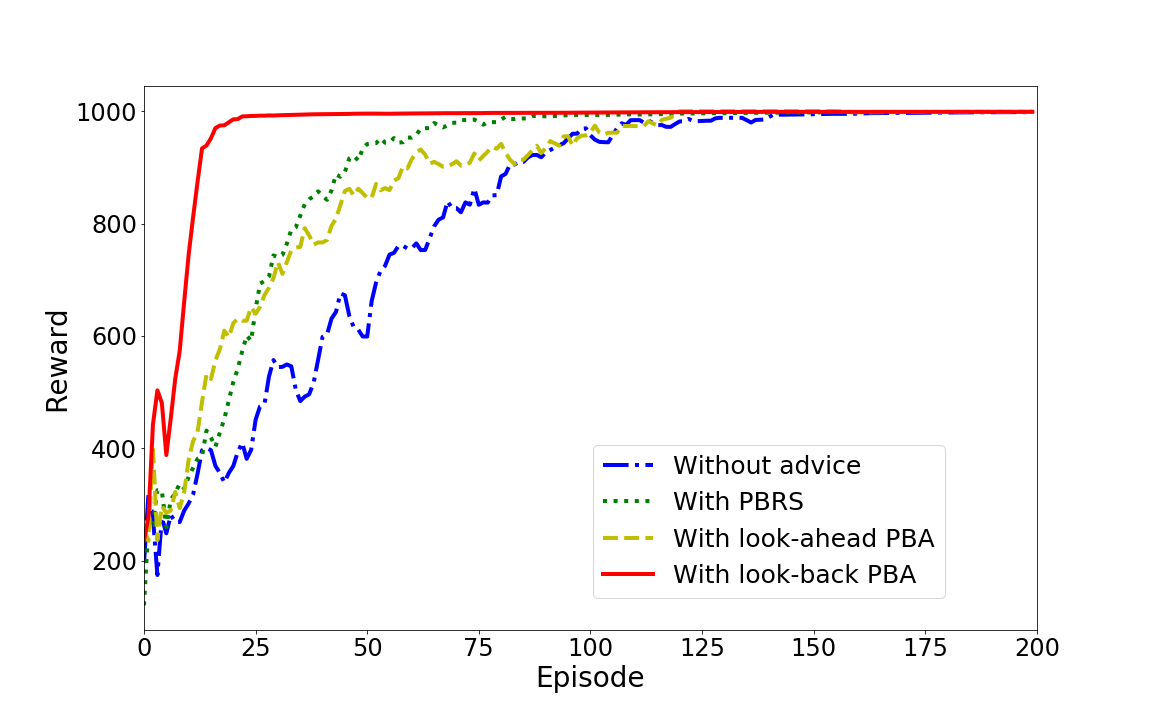}
	\caption{Average rewards in puddle-jump gridworld when jump success probability $p_j=0.2$. The baseline is the advantage actor-critic without advice.}\label{cliff_results}
\end{figure}
%And hence bad feedbacks will make the agent adhere to bad policies. 
%This can happen at the early steps of training and will impede the progress of leaning. 
%One solution to mitigate this problem is to combine PBA with other shaping methods. 
%Specifically, we combine PBRS with look-ahead PBA, by defining $\phi^{Mixed}(s,a):=\beta \phi^{PBRS}(s)+(1-\beta)\phi^{PBA}(s,a)$, where $\beta \in (0,1)$.
%As shown in Figure \ref{cliff_results}, this mixed advice outperforms schemes that use only PBRS or only look-ahead PBA. 
%This mixed version, however, does not perform as well as look-back PBA.

A smaller jump success probability $p_j$ is an indication that it is more difficult for the agent to reach the goal state $G$. 
%A smaller $p_j$ indicates that it is harder to reach $G$. 
%difficulty of the task to reach the goal can be roughly measured by jump success probability $p_j$, and a smaller $p_j$ means a harder task. 
Figure \ref{cliff_results2} shows that look-back PBA results in the highest reward for a more difficult task (lower $p_j$), when compared with the other reward shaping schemes.
\begin{figure}
	\centering
	\includegraphics[width=2.8 in]{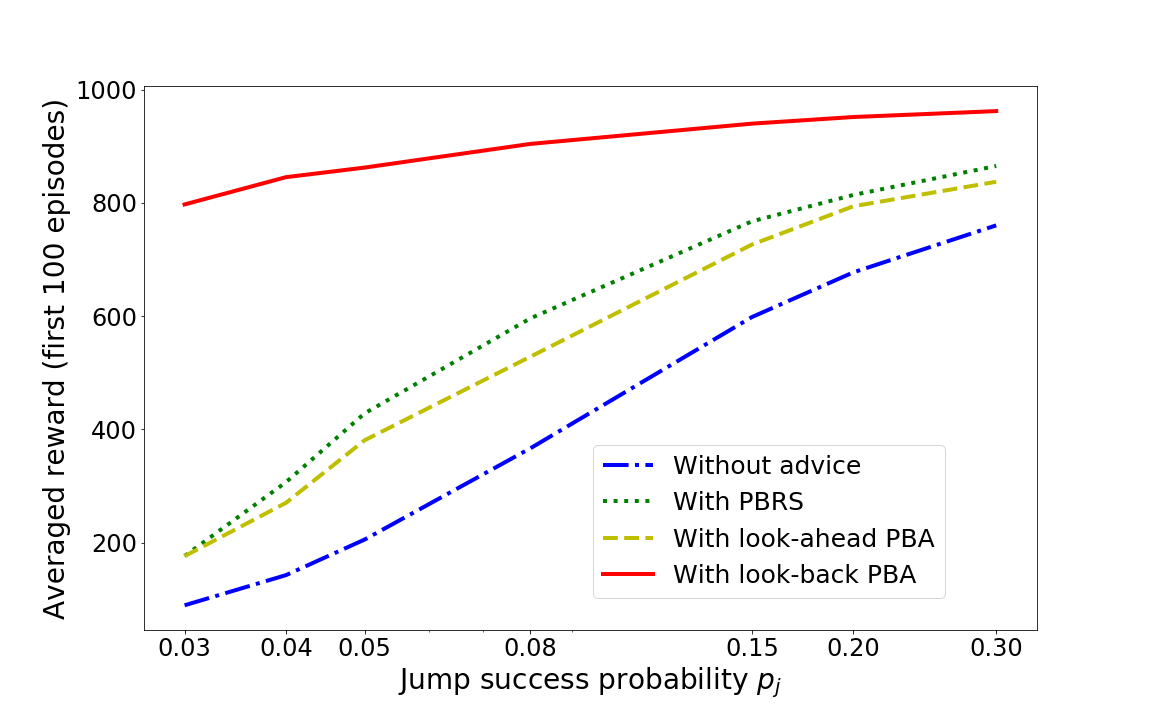}
	\caption{Average reward for the first 100 episodes with respect to the jump success probability $p_j$.}\label{cliff_results2}
\end{figure}
%that the harder the task is, the more reward gain can be achieved by look-back PBA or mixed advice.
%
\subsection{Continuous Mountain Car}

In the mountain car (MC) environment, an under powered car in a valley has to drive up a steep hill to reach the goal. In order to achieve this, the car should learn how to accumulate momentum. 
A schematic for this environment is shown in Figure \ref{MountainCarFig}.
\begin{figure}
	\centering
	\includegraphics[width=3 in]{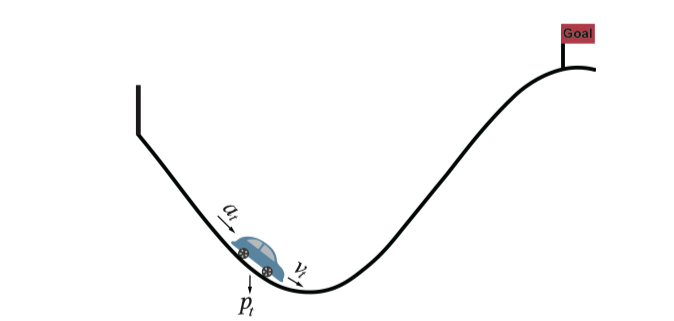}
	\caption{Schematic of the mountain-car environment. The agent's state is represented by its position $p_t$ (along the $x-$coordinate) and velocity $v_t$. The action $a_t$ is a force applied to the car. The goal is marked as a flag.}\label{MountainCarFig}
\end{figure}
This MC environment has continuous state and action spaces. 
The state $s=(p,v)$ denotes position $p \in [-1.2,0.6]$ and velocity $v \in [-0.07,0.07]$. 
The action $a \in [-1,+1]$. 
The continuous action space makes it difficult to use classic value-based methods, such as Q-learning and Sarsa-learning. 
%, since solving an optimization problem is required to obtain the best action at each state. This is one reason why we adopt actor-critic framework for this environment.
The reward provided by the environment depends on the action and whether the car reaches the goal. 
Specifically, once the car reaches the goal it receives $+100$, and before that, the reward at time $t$ is $-|a_t|^2$. 
This reward structure therefore discourages the waste of energy. 
This acts as a barrier for learning, because there appears to be a sub-optimal solution where the agent remains at the bottom of the valley.
Moreover, the reward for reaching the goal is significantly delayed, which makes it difficult for the conventional actor-critic algorithm to learn a good policy. 

One choice of a potential function while using PBRS in this environment is $\phi^{PBRS}(s_t):=p_t+2$, where the offset is so that the potential is positive. 
An interpretation of this scheme is: \emph{`state value is larger when the car is horizontally closer to the goal.'} 
The PBA scheme we use for this environment encourages the accumulation of momentum by the car-- the direction of the action is encouraged to be the same as the current direction of the car's velocity. 
In the meanwhile, we discourage inaction. 
Mathematically, the potential advice function has a larger value if \emph{ $a_t\neq 0$}. 
We let $\phi^{PBA}(s_t,a_t)=1$, if $a_tv_t >0$, and $\phi^{PBA}(s_t,a_t)=0$, otherwise.
%The function $\phi^{PBA}(s_t,a_t)$ is:% then: 
%%
%\begin{equation}\label{MC_PBA}
%\phi(s_t,a_t)=
%\begin{cases}
%1, & \text{if }
%\begin{aligned}[t]
%a_tv_t>0,
%\end{aligned}
%\\
%0, & \text{otherwise.}
%\end{cases}
%\end{equation}

In our experiments, we set $\gamma= 0.99$. 
To deal with the continuous state space, we use a neural network (NN) as a  function approximator. 
The policy distribution $\pi_{\theta}(a|s)$ is approximated by a normal distribution, the mean and variance of which are the outputs of the NN. 
The value function is also represented by an NN. 
We set $\alpha^{\theta}=1\times 10^{-5}$ and $\alpha^{\omega}=5.6\times 10^{-4}$, and use Adam \cite{Adam} to update the NN parameters. 
The results we report are averaged over 10 different environment seeds.%(seed 0 to 9 in OpenAI gym environment).
\begin{figure}
	\centering
	\includegraphics[width=2.8 in]{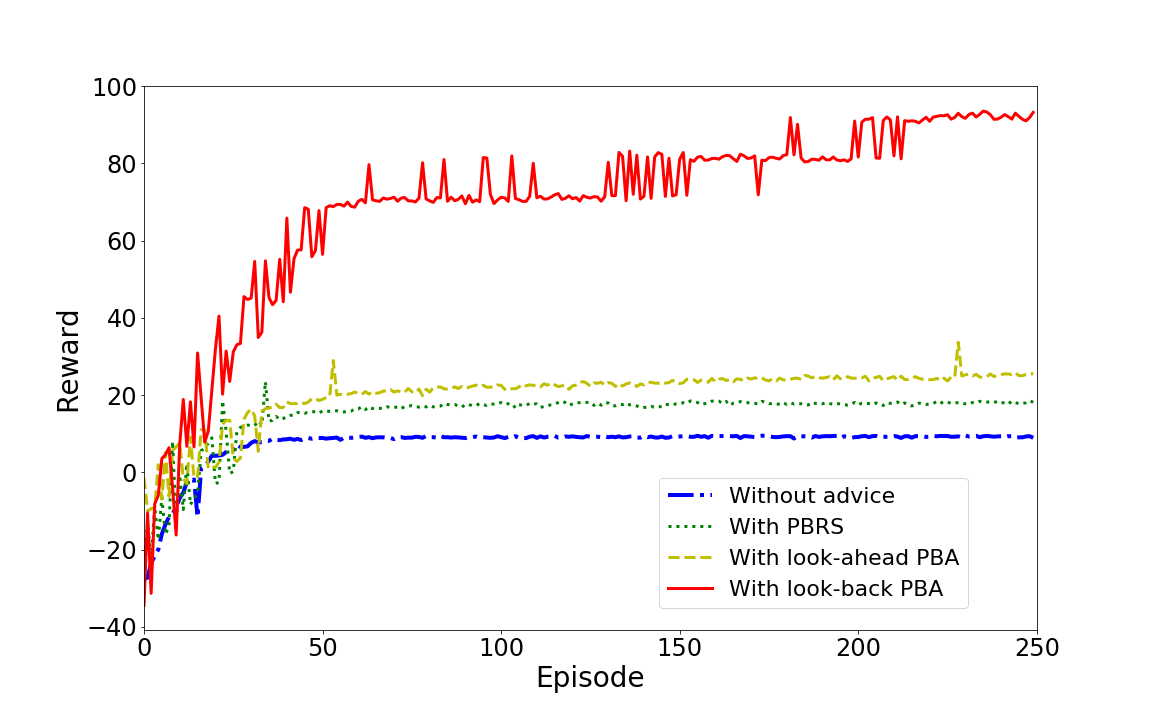}
	\caption{Average rewards for continuous mountain car problem (averaged over 10 different environment random seeds). The baseline is the A2C without advice.}\label{MountainCar_results}
\end{figure}
\begin{table}[h]
	\centering
	\begin{tabular}{|c | c | c | c |} 
		\hline
		No advice & PBRS & Look-ahead PBA & Look-back PBA \\ [0.5ex] 
		\hline\hline
		10\% & 20\% & 40\% & 100\% \\ 
%		 [1ex] 
		\hline
	\end{tabular}
	\caption{Percentage of trials where policy converges correctly in continuous mountain car problem.}
	\label{table:1}
\end{table}

Our experiments indicate that the policy makes the agent converge to one of two points: the goal, or remain stationary at the bottom of the valley. 
The percentage of solutions that converge to the goal is shown in Table \ref{table:1}. 
From Figure \ref{MountainCar_results} and Table \ref{table:1}, when learning with the vanilla A2C, the agent is able to reach the goal only in $10\%$ of the trials (out of 10 trials), and was stuck at the sub-optimal solution for the remaining trials. 
With PBRS, the agent could converge correctly in only $20\%$ of the trials. 
This is because the agent might have to take an action that moves it away from the goal in order to accumulate momentum. 
However, the potential function $\phi^{PBRS}(\cdot)$ discourages such actions. 
%This is due to the defectiveness of PBRS used during learning, and actually the agent should learn to back off which can sometimes be discouraged by the PBRS. 
In comparison, the average reward when using look-ahead PBA is slightly higher, but the agent is able to reach the goal in only $40\%$ of the trials.  Similar to the gridworld setup, look-back PBA performs the best, where the agent is able to reach the goal in $100\%$ of the trials.
\section{Conclusion}\label{ConclusionSection}

This paper presented a framework for augmenting the reward received by an RL agent with PBRS and with PBA. 
Different from prior work, we demonstrated that our approach can be used in environments with continuous states and actions, and when the optimal policy is stochastic. 
We presented guarantees on the convergence of an algorithm that augments an A2C architecture with these schemes. 
Our experiments indicated that these schemes allowed the agent to achieve higher average rewards, and learn an optimal policy faster. 
Future work will focus on establishing tighter bounds for Theorem \ref{convergence}, and extending our approach to the average reward case.
%
%PBRS/ PBA for the average reward case? This paper looks only at discounted reward. 
%
%tighter bounds for online actor-critic
%
%transfer learning using potential functions in new environments.
%
%learning potential-based function.
%
\bibliographystyle{IEEEtran}
\bibliography{AC_PBA}  % put name of your .bib file here
\end{document}